\documentclass[letterpaper,10pt,conference]{ieeeconf}
\IEEEoverridecommandlockouts
\usepackage{amsmath,amssymb,amsfonts}
\usepackage[pdftex]{graphicx}
\usepackage{textcomp}
\usepackage{xcolor}
\usepackage{inconsolata}
\usepackage{soul}

\usepackage[bookmarks=true]{hyperref}

\let\labelindent\relax
\usepackage{enumitem}

\usepackage{float}
\usepackage{listings}
\usepackage{multirow}
\usepackage{xparse}
\usepackage{optidef}
\usepackage{algorithm}
\usepackage[noend]{algpseudocode}

\makeatletter
\let\NAT@parse\undefined
\makeatother
\usepackage[numbers,sort,compress]{natbib}

\usepackage[all=normal,bibbreaks=tight,paragraphs=tight,floats=tight]{savetrees}

\hypersetup{
colorlinks=true,
urlcolor=blue,
}

\definecolor{mygreen}{rgb}{0,0.6,0}
\definecolor{mypurple}{rgb}{0.6,0,0.6}
\definecolor{myblue}{rgb}{0,0,0.9}
\newtheorem{theorem}{Theorem}[section]

\ifx\arxiv\undefined

\else

\fi

\NewDocumentCommand \proposition {g g g g} {\texttt{#1}(#2
  \IfValueTF{#3}{,\,#3}{}
  \IfValueTF{#4}{,\,#4}{}
  )
}

\newcommand{\pddl}[1]{\lstinline[basicstyle=\small\ttfamily,breaklines=true,commentstyle={}]{#1}}

\lstdefinelanguage{pddlaction}{
    deletecomment=[l]{;}
}
\newcommand{\pddlaction}[1]{\lstinline[language=pddlaction,basicstyle=\small\ttfamily,breaklines=true]{#1}}

\NewDocumentCommand \actioncall {g g g g} {\text{#1}(
  \IfValueTF{#2}{#2}{}
  \IfValueTF{#3}{,#3}{}
  \IfValueTF{#4}{,#4}{}
  )
}

\usepackage[font=footnotesize,justification=justified,singlelinecheck=false]{caption}
\usepackage{etoolbox}
\makeatother
\def\tablename{Table}

\author{Brandon Vu, Toki Migimatsu, Jeannette Bohg
\thanks{The authors are with the Department of Computer Science, Stanford University (e-mail: \{vubt, takatoki, bohg\}@cs.stanford.edu).}
\thanks{Toyota Research Institute provided funds to support this work.}
}

\begin{document}
\title{COAST:\\ Task and Motion Planning with PDDL \textbf{CO}nstraints \textbf{A}nd \textbf{ST}reams}
\title{\textbf{COAST}: \textbf{CO}nstraints \textbf{A}nd \textbf{ST}reams for Task and Motion Planning}
\maketitle




\begin{abstract}
Task and Motion Planning (TAMP) algorithms solve long-horizon robotics tasks by integrating task planning with motion planning; the task planner proposes a sequence of actions towards a goal state and the motion planner verifies whether this action sequence is geometrically feasible for the robot. However, state-of-the-art TAMP algorithms do not scale well with the difficulty of the task and require an impractical amount of time to solve relatively small problems. We propose Constraints and Streams for Task and Motion Planning (COAST), a probabilistically-complete, sampling-based TAMP algorithm that combines stream-based motion planning with an efficient, constrained task planning strategy. We validate COAST on three challenging TAMP domains and demonstrate that our method outperforms baselines in terms of cumulative task planning time by an order of magnitude. You can find more supplementary materials on our project \href{https://branvu.github.io/coast.github.io}{website}.
\end{abstract}


\section{Introduction}
\label{sec:intro}

We aim to equip a robot with the ability to solve complex long-horizon tasks that require a combination of symbolic and geometric reasoning. {\em Task and Motion Planning\/} (TAMP) is an approach for solving such tasks. TAMP methods often use task planning to produce a sequence of symbolic actions, i.e. a task plan, in addition to using sampling-based motion planning to ensure the task plan is geometrically feasible. If the task plan is geometrically infeasible, then this result needs to be communicated to the task planner for replanning. Two main paradigms of communication exist: \textit{sample-first} and \textit{plan-first} \cite{garrett2021integrated}. \textit{Sample-first} methods perform motion sampling first (without any task plan) and then query task planning to sequence only the geometrically feasible samples \cite{factor, pddlstream}. \textit{Plan-first} methods perform task planning first and then refine the task plans with motion sampling, where sampling failures due to geometric infeasibility are translated into task planning constraints for the next iteration of task planning \cite{idtmp, tmit, neuro}.

Two sampling-based TAMP algorithms closely related to our work are PDDLStream \cite{pddlstream} and Iteratively Deepened Task and Motion Planning (IDTMP) \cite{idtmp}. PDDLStream is an optimistic \textit{sample-first} method that breaks down motion planning into black-box sampling functions called streams and integrates them into the task-planning problem as objects and action preconditions in the {\em Planning Domain Definition Language\/} (PDDL) \cite{pddl}. A limitation of PDDLStream is that it must generate the symbolic objects to be used for task planning without knowing which ones may be necessary. But generating too many objects results in exponential task planning times. IDTMP is a \textit{plan-first} method that treats task planning as Constraint Satisfaction Problems (CSPs) and uses constraints to communicate motion planning failures. A limitation of IDTMP is that it requires a discretization of the workspace which prevents IDTMP from applying to domains with large workspaces like our \textit{Kitchen} and \textit{Rover} domains.   

We propose a probabilistically-complete, \textit{plan-first} TAMP algorithm that is significantly faster than PDDLStream and IDTMP. This speedup occurs by using a direct stream planning algorithm to create stream objects \textit{after} task planning rather than \textit{before} to avoid the computational cost of task planning with many unnecessary stream objects. 
We validate our method on three TAMP domains (Fig.~\ref{fig:teaser}), each one challenging in different ways, and demonstrate that our method outperforms both PDDLStream and IDTMP by an order of magnitude in task planning time. 

\begin{figure}
    \centering
    \includegraphics[width=\columnwidth]{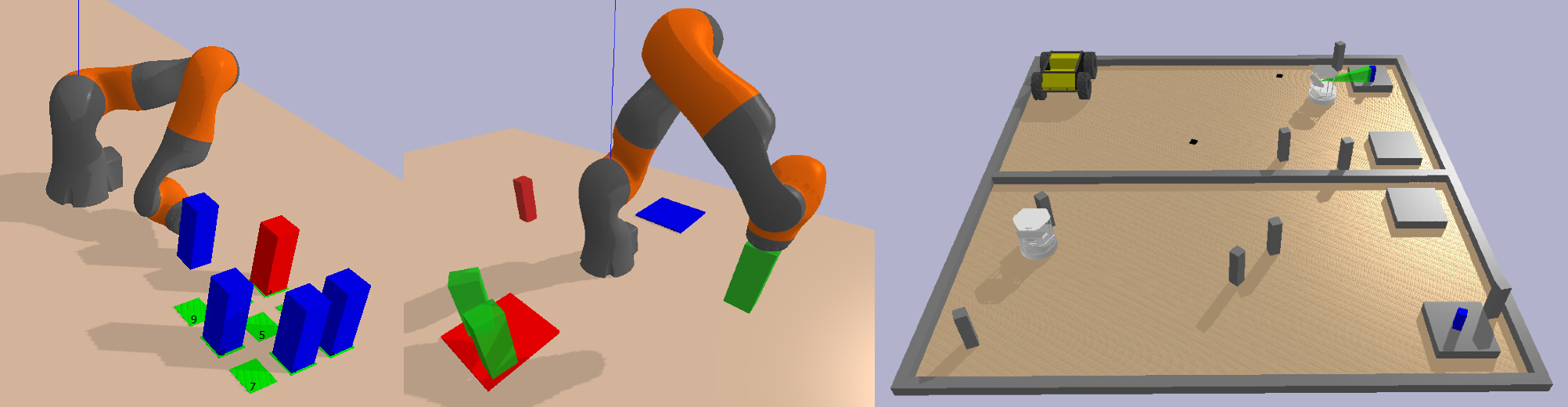}
    \caption{We propose COAST, a sampling-based TAMP algorithm that is able to solve complex, geometrically constrained, long-horizon planning problems faster than prior state-of-the-art. 
    We demonstrate the ability of our algorithm to solve problems from three domains: a $3 \times 3$ grid rearranging task (\textit{Blocks} \cite{idtmp}, \textit{left}), a constrained pick-and-place kitchen task (\textit{Kitchen} \cite{pddlstream}, \textit{middle}), a rover surveillance task with obstacles (\textit{Rover} \cite{pddlstream}, \textit{right}).
    } 
    \label{fig:teaser}
\end{figure}


\section{Related Works}
\label{sec:related}
TAMP problems are challenging due to the need to search over both discrete and continuous spaces. Thus, many works propose different techniques to reduce the search complexity of TAMP. Hierarchical Task Networks (HTN)~\cite{htn} are a class of TAMP algorithms that use expert-designed hierarchies to reduce the dimension of the search problem~\cite{hatp}. Other works \cite{Backward, Wolfe2010CombinedTA, Karlsson2012CombiningTA} extend this approach for manipulation. Instead of using hierarchies, Srivastava et al. \cite{srivastava2014combined} is a plan-first framework that uses constraints defined in PDDL to prune infeasible plans. We also use constraints and a plan-first approach, but unlike this prior work, our method is applicable outside of manipulation and does not assume the absence of reachable dead-end states. We accomplish this by using streams as an abstraction for motion planning while also using queue-based algorithms that revisit previous states, relaxing the dead-end state assumption. 
Our focus lies in increasing the speed of universally applicable TAMP algorithms, agnostic to the motion planning implementation. To achieve this, we advance the integration of classical task planning \cite{STRIPS} in PDDL \cite{pddl} and stream-based motion planning \cite{pddlstream}.

The motion planning component of TAMP commonly consists of finding valid robot trajectories for manipulating objects while avoiding collisions in addition to finding satisfiable geometric assignments such as grasps and poses \cite{motion}. Optimization-based approaches~\cite{LGP, migimatsu2020object} attempt to find optimal motion plans with nonlinear optimization, which can be sensitive to initial conditions and prone to failure. Thomason et al. propose TMIT* \cite{tmit} which plans in a hybrid symbolic and continuous state space, using CSP constraints and asymmetric forward and reverse motion planning. Our framework provides a connection from constraint-based planning to streams, which could be extended to an approach like TMIT*. 

In contrast to refining each symbolic action sequentially, recent works propose to break down each motion planning sub-problem into even smaller, reusable functions that enable more efficient motion planning across actions. 
PDDLStream~\cite{pddlstream} proposes a general TAMP framework that formalizes these as \textit{streams}. Decomposing the motion planning problem into these lightweight samplers gives rise to efficient sampling algorithms that can intelligently resample streams until a task plan is refined. However, incorporating streams directly into task planning causes the PDDL problem to exponentially grow in the number of objects, which slows down task planning. We propose a stream-based TAMP algorithm that benefits from efficient motion planning while keeping task planning light.

There has been a rise in learning-based methods that seek to overcome the disadvantages of classical TAMP. Driess et al. \cite{dvr} trains a policy to solve TAMP tasks from images and demonstrations produced by classical TAMP solvers. Other works augment classical TAMP solvers by accelerating planning with learned heuristics~\cite{silver2021planning, wang2021learning, learnedheuristic} or giving them the ability to handle uncertainty~\cite{m0m, garrett2020online}. Our work can benefit these methods by speeding up solving times and increasing the scale of possible tasks.


PDDLStream \cite{pddlstream} and IDTMP \cite{idtmp} are the two works most closely related to ours. We build off of PDDLStream's stream framework for motion planning and propose a constrained task-planning method similar to that of IDTMP.

\section{Background}
\label{sec:background}
\subsection{Planning Domain Description Language}

Our framework uses the {\em Planning Domain Definition Language\/} (PDDL) \cite{pddl} for task planning. A PDDL domain, typically defined as a \texttt{domain.pddl} file, can be described as a tuple $(\Phi, \mathcal{A})$, where $\Phi$ is the set of predicates (binary-valued properties of objects) and $\mathcal{A}$ is the set of actions. A PDDL problem, typically defined as a \texttt{problem.pddl} file, is a tuple $(\mathcal{O}, s_0, g)$, where $\mathcal{O}$ is the set of environment objects, $s_0$ is the initial state, and $g$ is the goal formula to be satisfied. The task planner's role is to find a sequence of actions, or a task plan $\pi$, that will turn the initial state $s_0$ into a new state that satisfies the goal formula $g$.

Actions are defined by their preconditions---a logic formula that must be satisfied by the current state in order to execute the action---and effects---a formula that describes how the state changes upon executing the action. The following are PDDL definitions of \pddlaction{Pick} and \pddlaction{Place} actions for a simple pick-and-place domain that we will use as a running example in this paper.
\begin{lstlisting}[label=code:pddl]
(:action Pick
  :parameters (?o - obj ?r - region)
  :precondition (and
    (on ?o ?r)
    (handempty))
  :effect (and
    (not (on ?o ?r))
    (holding ?o)
    (not (handempty))))
(:action Place
  :parameters (?o - obj ?r - region)
  :precondition (holding ?o)
  :effect (and 
    (on ?o ?r)
    (not (holding ?o))
    (handempty)))
\end{lstlisting}

\subsection{Streams}

Our framework uses streams from PDDLStream \cite{pddlstream} to perform sampling-based motion planning. Streams decompose the long-horizon motion planning problem into unit sampling functions that each address a small component of the motion planning problem. 

A stream is defined as a generator function $\sigma(x_1, \dots, x_n) \rightarrow (y_1, \dots y_m)$ which takes a tuple $x_1, \dots, x_n$ as input and generates a tuple of outputs $(y_1, \dots, y_m)$. For example, the \pddl{sample-pose} stream may take in an object $o$ and region $r$ and output a placement pose $p$ for $o$ in $r$:
\begin{lstlisting}
(:stream sample-pose
  :inputs (?o - obj ?r - region)
  :outputs (?p - pose))
\end{lstlisting}
Streams are typically defined in a \texttt{streams.pddl} file and are accompanied by user-provided Python functions which implement the actual sampling procedure. In this paper, we refer to the set of streams in the PDDLStream domain as $\Sigma$. We extend streams by adding a cache and a probability to return a cached result with a probability from $[0, 1)$. We only use this feature for the \textit{Rover} domain and ablate it in Sec.~\ref{sec:experiments}.

A stream \textit{instance} is a stream instantiated with concrete inputs and outputs, such as \pddl{sample-pose(apple, shelf)}$\rightarrow$\pddl{p1}. Outputs of stream instances are called stream objects, and the inputs to stream instances can either be stream objects or PDDL objects. In this case, \pddl{apple} and \pddl{shelf} are both PDDL objects, while \pddl{p1} is a stream object. PDDL objects are defined in the PDDL problem as object set $\mathcal{O}$, but stream objects do not exist at the beginning of task planning and need to be created during planning.

Every time a stream instance is called, it may generate new values for the output stream objects. For example, \pddl{sample-pose} might generate poses for \pddl{p1} where the positions are sampled randomly from the support area of \pddl{shelf}. A stream instance also returns a ``certified fact", or a symbolic proposition, along with the sampled values to certify the success of sampling. For example, if \pddl{sample-pose} succeeds, it would output the certified fact \pddl{(sample-pose apple shelf p1)}, which indicates that \pddl{p1} is a valid pose for \pddl{apple} on \pddl{shelf}. However, if sampling fails---for example, because \pddl{shelf} is too small to support a placement pose for \pddl{apple}---then the certified fact is not returned.

Note that our definition of certified facts is slightly different from that of PDDLStream; while stream instances in PDDLStream may output multiple certified facts, we require that stream instances output a single certified fact containing all of the input and output objects so that there is a bijective mapping between stream instances and certified facts. This is not a restriction because non-bijective certified facts are instead added to the geometric postconditions of our action. The bijective mapping allows our stream planner to directly determine the required stream instances from a set of desired certified facts. PDDLStream, on the other hand, must blindly create stream instances until all desired certified facts are covered, which is an expensive iterative process and often results in the creation of many unused certified facts in the task state.

A stream plan $\psi$ is a sequence of stream instances that each must be sampled successfully to complete the motion planning problem for a candidate task plan $\pi$. After computing a stream plan for a candidate action skeleton, the last step is to sample the streams to generate a motion plan. If a stream fails, we resample the streams until the entire stream plan is successful. We also need to decide when to give up motion planning for a candidate action skeleton and mark it as infeasible. In our experiments, we use the semi-complete \texttt{Adaptive} PDDLStream algorithm to handle stream plan sampling and termination. We refer readers to \cite{pddlstream} for an in-depth description of this algorithm.

The key difference between PDDLStream and our method is how stream objects and certified facts are treated. PDDLStream treats stream objects as PDDL objects and allows certified facts to be used in the preconditions of PDDL actions. The main disadvantage of this approach is that the task planner is now required to decide what stream objects to use for a task plan. It is not known a priori what stream objects are required to satisfy the task planning goal. Therefore, stream generation in PDDLStream is an iterative process where stream objects are incrementally introduced by level to the PDDL problem until the task planner succeeds. At first, when no stream objects are available, the task planner will certainly fail. As the number of stream objects grows, task planning quickly becomes intractable due to its PSPACE-hard complexity. Task planning is therefore a significant bottleneck in PDDLStream when many stream objects are required, which may happen for problems that have many movable objects and require long task plans. Our key insight is that deciding what stream objects to use for a task plan can be done with a simple stream planning procedure  (Sec.~\ref{sec:gpddl-stream-planning}) that does not require a PDDL solver. The integration between task and motion planning is achieved with PDDL constraints (Sec.~\ref{sec:gpddl-task-planning}) rather than deferring the stream instance to a later level like in PDDLStream.

\section{COAST Algorithm}
\label{sec:methods}
The PDDL domain $(\Phi, \mathcal{A})$ and problem $(\mathcal{O}, s_0, g)$ given to COAST are defined with vanilla PDDL (i.e. no stream objects), with actions resembling the example definitions of \pddlaction{Pick} and \pddlaction{Place} in Sec.~\ref{code:pddl}. To connect the PDDL domain with streams $\Sigma$, we introduce a stream planning layer that plans using geometric actions $\mathcal{A}_{\text{geom}}$ and the initial geometric state $s_{0_\text{geom}}$, explained in more detail in Sec.~\ref{sec:gpddl-stream-planning}.
Then, COAST enters a loop that alternates between task and motion planning. An overview of our algorithm is provided in Alg.~\ref{alg:gpddl}.
We use an off-the-shelf PDDL solver to propose a candidate task plan $\pi$ that satisfies the symbolic goal but is not necessarily valid from a motion planning standpoint (Line~\ref{line:task-plan}). We then run a novel stream planning method to ground the task plan $\pi$ with stream objects to produce a geometric task plan $\pi_{\text{geom}}$ (Line~\ref{line:stream-plan}). For example, if $\pi$ is the task plan [\pddl{Pick(apple, table)}, \pddl{Place(apple, rack)}], then $\pi_{\text{geom}}$ might look like [\pddlaction{Pick(apple, table; g1)}, \pddlaction{Place(apple, rack; g1, p1)}], where \pddlaction{Pick} and \pddlaction{Place} are grounded with stream objects necessary for motion planning, such as grasp \pddl{g1} and pose \pddl{p1}. Stream planning also produces a stream plan $\psi$, which is a sequence of stream instances that need to be sampled to produce values for the stream objects in $\pi_{\text{geom}}$. We then use PDDLStream's \texttt{Adaptive} algorithm to sample the stream plan and return a binding map $Y$ from the stream outputs $y$ to their sampled values along with the set of certified facts $s_\psi$ (Line~\ref{line:adaptive}). If there is one certified fact per stream instance in $\psi$, then the entire stream plan was sampled successfully and we can terminate planning. Otherwise, we constrain the PDDL problem by modifying the initial state $s_0$ and action definitions in $\mathcal{A}$ to force the PDDL solver to find an alternative plan (Line~\ref{line:constrain-pddl}). Then the planning cycle continues until we successfully complete motion planning for a task plan or time out. We maintain a queue of all previous planning states, and we sort the queue by the frequency of the set of constraints and the number of constraints to prioritize more unique and less constrained task states. We prove our algorithm is probabilistically complete in Sec. \ref{sec:prob-proof}. The following subsections describe our stream planning procedure and task planning constraints in more detail.


\begin{algorithm}[tb]
    \caption{COAST TAMP Algorithm Overview}
    \begin{algorithmic}[1]
        \Function{COAST}{$\Phi, \mathcal{A}, \mathcal{A}_{\text{geom}}, \Sigma, \mathcal{O}, s_0, s_{0_{\text{geom}}}, g$}
            \State $\mathcal{Q} \gets []$ 
            \State $\Call {PUSH}{\mathcal{Q}, \Phi, \mathcal{A}, \mathcal{O}, s_0, g}$
            \While{not \Call{Timeout}{}}
                \State $\Call {POP}{\mathcal{Q}, \Phi, \mathcal{A}, \mathcal{O}, s_0, g}$
                \State $\pi \gets \Call{TaskPlan}{\Phi, \mathcal{A}, \mathcal{O}, s_0, g}$ \label{line:task-plan}
                \If{$\pi = \text{None}$ \textbf{and} $Q = []$}
                    \State \Return \text{None}
                \EndIf

                \State $\pi_{\text{geom}}, \psi \gets \Call{StreamPlan}{\mathcal{A}_{\text{geom}}, \Sigma, s_{0_{\text{geom}}}, \pi}$\label{line:stream-plan}
                \State $Y, s_\psi \gets \Call{AdaptiveBinding}{\psi}$\label{line:adaptive}
                
                \If{\Call{IsSuccessful}{$\psi, s_\psi$}}
                    \Return $\pi_{\text{geom}}, Y$
                \EndIf
                \State $\Call {PUSH}{\mathcal{Q}, \Phi, \mathcal{A}, \mathcal{O}, s_0, g}$
                \State $s_0, \mathcal{A} \gets 
                \Call{ConstrainPDDL}{\mathcal{A}, \psi, s_\psi, s_0}$\label{line:constrain-pddl}
                \State $\Call {PUSH}{\mathcal{Q}, \Phi, \mathcal{A}, \mathcal{O}, s_0, g}$
            \EndWhile
        \EndFunction
    \end{algorithmic}
    \label{alg:gpddl}
\end{algorithm}

\subsection{COAST Stream Planning}
\label{sec:gpddl-stream-planning}

\begin{algorithm}[tb]
    \caption{COAST Stream Planning}
    \begin{algorithmic}[1]
        \Function{StreamPlan}{$\mathcal{A}_{\text{geom}}, \Sigma, s_{0_{\text{geom}}}, \pi$}
            \State $\psi \gets \varnothing$
            \For{$t = 1 \dots H$}
                \State $a_t \gets \pi[t]$
                \State $a_{t_{\text{geom}}} \gets \Call{GetGeomAction}{\mathcal{A}_{\text{geom}}, a_t}$
                \State $\overline{a}_{t_{\text{geom}}} \gets \Call{GroundGeomAction}{a_{t_{\text{geom}}}, s_{t-1_{\text{geom}}}}$
                \State $\psi \gets \psi \cup \Call{GetPreconditionStreams}{\overline{a}_{t_{geom}}}$
                \State $s_{t_{\text{geom}}} \gets \Call{ApplyGeomAction}{s_{t-1_{\text{geom}}}, \overline{a}_{t_{\text{geom}}}}$
            \EndFor
            \State \Return $\psi$
        \EndFunction
    \end{algorithmic}
    \label{alg:streamplan}
\end{algorithm}

For every candidate task plan produced by the task planner, we need to perform motion planning to produce a trajectory for the robot to execute the task plan, or if motion planning fails, mark the task plan as infeasible. We define the motion planning problem as finding satisfiable assignments to stream objects in a given stream plan. Our method uses a novel stream planning subroutine to generate a stream plan from a candidate task plan. The pseudocode for this subroutine is provided in Alg.~\ref{alg:streamplan}.

Each action is associated with a set of streams that need to be executed during motion planning. We specify how these streams are executed for each action in a separate \texttt{geometric.pddl} file. For example, we may define the geometric \pddl{Place} action as:
\begin{lstlisting}
(:geom-action Place
  :parameters (?o - obj ?r - region)
  :inputs (?g - grasp)
  :outputs (?p - pose)
  :geom-precondition (and
    (in-grasp ?o ?g)
    (sample-pose ?o ?r ?p))
  :geom-effect (and
    (not (in-grasp ?o ?g))
    (at-pose ?o ?p)))
\end{lstlisting}

The \pddl{:parameters} field defines the PDDL object parameters for this action; it should be equivalent to \pddl{:parameters} defined for the corresponding PDDL action in \texttt{domain.pddl}. The \pddl{:inputs} and \pddl{:outputs} fields define the input and output stream objects for this action. This \pddlaction{Place} action, for example, takes as input a grasp \pddl{?g} and outputs a pose \pddl{?p}. While the \pddl{:parameters} will be determined by the task planner (e.g. \pddlaction{Pick(apple, table)}), the stream objects need to be determined during the stream planning phase.

During stream planning, we maintain a geometric state, which, similarly to the symbolic state in PDDL, is represented with a set of geometric propositions. While symbolic propositions like \pddl{(on apple table)} are defined with symbolic objects, geometric propositions can also be defined with stream objects, such as \pddl{(at-pose apple p1)}. The \pddl{:geom-precondition} field defines the requirements for applying a geometric action to a geometric state, and the \pddl{:geom-effect} field specifies how the geometric state changes upon executing each action---just like the preconditions and effects of symbolic actions in PDDL.

The job of stream planning is two-fold: 1) grounding each action in a given task plan with stream objects, and 2) computing a stream plan, or an ordered sequence of stream instances from the grounded actions.

\subsubsection{Grounding geometric actions with stream objects}

The \pddl{:inputs} are determined by using the \pddl{:geom-precondition} field to find matching stream objects in a geometric state at a specific step in the plan. For example, a precondition for \pddlaction{Place(apple, rack)} is \pddl{(in-grasp apple ?g)}, where \pddl{?g} is an undetermined stream object defined in the \pddl{:inputs} field. The geometric state is a set of geometric propositions. If the geometric state at the beginning of \pddlaction{Place(apple, rack)} is \{\pddl{(at-pose orange p1)}, \pddl{(in-grasp apple g1)}\}, then from the precondition \pddl{(in-grasp apple ?g)}, we can infer that \pddl{g1} is a valid argument for the input parameter \pddl{?g}. The \pddl{:outputs} are generated by the actions, so the stream planner will simply define new stream objects for each action call. For example, the stream planner may define a stream object \pddl{p1} as the output of \pddlaction{Place(apple, rack)}. The action call \pddlaction{Place(apple, rack)} is now grounded with concrete stream objects \pddl{g1} and \pddl{p1} (these stream objects will not be assigned values until the stream sampling stage). We will represent this grounded action call as \pddlaction{Place(apple, rack; g1, p1)}.

\subsubsection{Computing a stream plan from grounded actions}

Another geometric precondition of \pddlaction{Place(apple, rack)} is the certified fact \pddl{(sample-pose apple rack ?p)}. When the geometric action is fully grounded with stream objects (e.g. \pddlaction{Place(apple, rack; g1, p1)}), then the certified facts in its preconditions can be mapped to stream instances. For example, the \pddl{sample-pose} precondition becomes \pddl{(sample-pose apple rack p1)}, which corresponds to the stream instance \pddl{sample-pose(apple, rack)}$ \rightarrow $\pddl{p1}. This precondition indicates that the successful sampling of the stream instance \pddl{sample-pose(apple, rack)}$ \rightarrow $\pddl{p1} is required for the execution of the geometric action \pddlaction{Place(apple, rack; g1, p1)}, and thus this stream instance is added to the stream plan $\psi$.

Note that during the stream planning phase, the planned stream instances are not actually sampled. The evaluation of stream instances are deferred to the stream sampling phase (e.g. PDDLStream's \texttt{Adaptive} algorithm). During stream sampling, if the stream instance \pddl{sample-pose(apple, rack)}$ \rightarrow $\pddl{p1} produces a successful sample, then the certified fact \pddl{(sample-pose apple rack p1)} gets returned, and the corresponding precondition for the geometric action \pddlaction{Place(apple, rack; g1, p1)} is satisfied. Otherwise, stream sampling continues until timeout, and the task planning domain will be updated with a constraint from the most recent sampling failure.

\subsection{COAST Constraints}
\label{sec:gpddl-task-planning}
Our approach relies on a constraint-based feedback model where motion planning failures during plan refinement are converted to logical constraints embedded in the task planner. We observed little to no slowdown in the task-planning time from adding these constraints to the planner. We provide generalized sequence, action, and collision constraints. 

\subsubsection{Sequence Constraint}
If the task plan [\pddlaction{Pick(apple, table)}, \pddlaction{Place(apple, rack}] fails because \pddlaction{Place(apple, rack)} is infeasible, then we need the task planner to produce an alternative task plan where \pddlaction{Place(apple, rack)} is not attempted directly after the same preceding sequence of actions [\pddlaction{Pick(apple, table)}]. This is a general constraint that requires automatically augmenting \texttt{domain.pddl} with timestamps in order to be compatible with off-the-shelf PDDL solvers. The implementation of this constraint is described in detail in the supplementary material on our webpage.



\subsubsection{Action Constraint}
It may be appropriate to prevent an action being executed with the same arguments, regardless of the sequence of actions preceding it. We can accomplish this by automatically augmenting every action in \texttt{domain.pddl} with the precondition that it has not failed before. Below, we show the augmented definitions of \pddlaction{Pick} and \pddlaction{Place}, where the original preconditions from the definitions in Sec.~\ref{code:pddl} are replaced with {\color{mygreen} \pddl{; ...same as original}} for brevity.
\begin{lstlisting}[label=code:augmentedpddl]
(:action Pick
  :parameters (?o - obj ?r - region)
  :precondition (and
    ; ...same as original
    (not (fail-pick ?o ?r))))
(:action Place
  :parameters (?o - obj ?r - region)
  :precondition (and
    ; ...same as original
    (not (fail-place ?o ?r))))
\end{lstlisting}

Suppose we have the task plan [\pddlaction{Pick(apple, table)}, \pddlaction{Place(apple, rack)}]. If motion planning for this task plan fails on the first action \pddlaction{Pick(apple, table)}, then we can prevent the task planner from proposing this action again by adding the \pddl(fail-pick apple table) proposition to the initial state $s_0$ in the PDDL problem. Since we maintain a queue of all previous planning states and their constraints, we will eventually revisit previous plans and maintain probabilistic completeness (Sec. \ref{sec:prob-proof}). Note that this constraint is analogous to the failure-generalization constraint in IDTMP \cite{idtmp}.

\subsubsection{Collision Constraint}
Sequence and action constraints can be applied to any domain, but for manipulation domains, we may often want a stronger constraint for handling movable obstructions. IDTMP~\cite{idtmp} proposes a location-based constraint for the \textit{Blocks} domain, where if picking or placing a particular block at a particular location is identified to cause a collision, then \textit{all} blocks will be prohibited from being picked or placed at the same location. Similar to IDTMP, this constraint relies on a discrete set of pre-defined locations. We show an abbreviated action definition below for the \textit{Block} domain.
\begin{lstlisting}[label=code:constraintpddl]
(:action Pick
  :parameters (?b - block ?l - loc)
  :precondition (and 
    ; ...
    (clear ?b))
  :effect (and 
    ; ...
    (not (clear ?b)) (clear ?l)))
(:action Place
  :parameters (?b - block ?l - loc)
  :precondition (and 
    ; ... 
    (clear ?l))
  :effect (and 
    ; ... 
    (clear ?b) (not (clear ?l)))
)
\end{lstlisting}
We formulate this constraint as a logical implication and append the implication to the effect of the action that occurs before the failure.



Let $a$ be the action that failed and $l$ be the location of a collision. Our collision constraint is then the following: 
\[\neg \texttt{clear}(l) \implies \texttt{fail-}a \wedge \texttt{clear}(l) \implies \neg \texttt{fail-}a\]
This means that when an action fails, we will prune out any plan that has the same collision.
\section{Experiments}
\label{sec:experiments}

\begin{figure}
    \centering
    \includegraphics[width=\columnwidth]{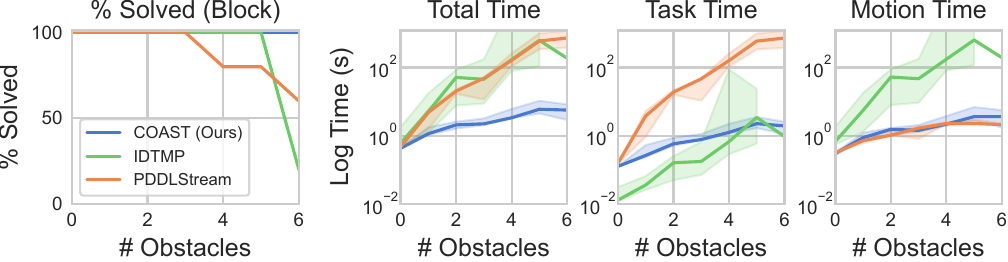}
    \caption{Percentage solved and cumulative task and motion planning times for the \textit{Blocks} domain with increasing number of obstacles.
    On the most complex configuration (6 obstacles), our algorithm achieves 100\% success while IDTMP achieves 20\% and PDDLStream achieves 60\%. The reported planning times include the failed trials that time out at 1200s. Our algorithm solves the largest problem two orders of magnitude faster than PDDLStream and IDTMP.
    }
    \label{fig:blocks}
\end{figure}
\begin{figure}
    \centering
    \includegraphics[width=\columnwidth]{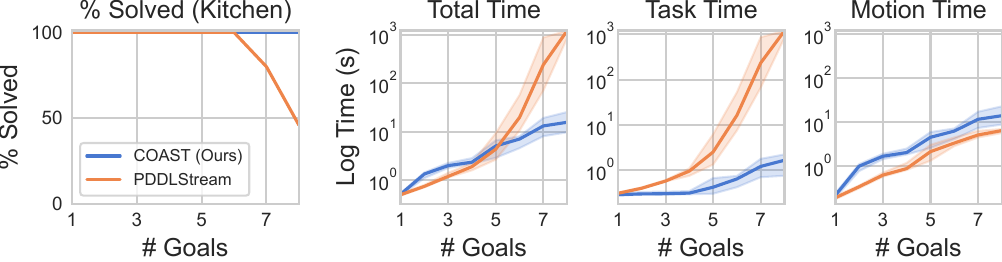}
        \caption{Percentage solved and cumulative task and motion planning times for the \textit{Kitchen} domain with increasing number of cook/clean goals. 
        PDDLStream's planning process times out after 1200 seconds for 46\% of the most challenging tasks (8 cook/clean goals), whereas our method achieves 100\% success at magnitudes faster. PDDLStream's slow task planning times come from the explosive growth of its task state with stream objects, which our method avoids by introducing stream objects \textit{after} task planning.
        }
    \label{fig:kitchen}
\end{figure}
\begin{figure}
    \centering
    \includegraphics[width=\columnwidth]{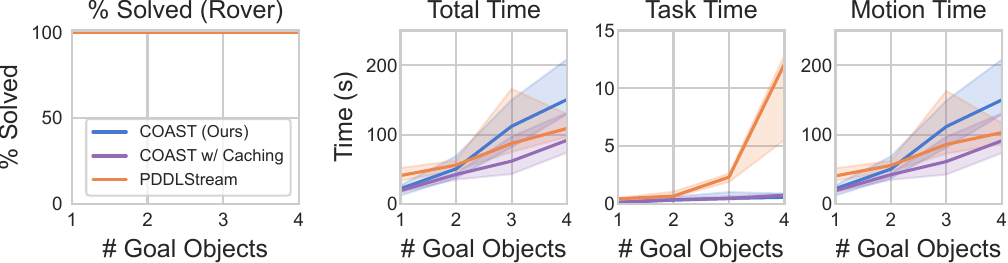}
        \caption{Percentage solved and cumulative task and motion planning times for the \textit{Rover} domain with increasing number of goal objects. We also include an ablation (purple) of our method with stream instance caching. PDDLStream has exponential growth in task planning time, whereas our task planning time remains nearly constant. 
        }
    \label{fig:rover}
   
\end{figure}
\paragraph{Experimental Domains, Metrics and Baselines}
We evaluate our approach in the three domains (Blocks~\cite{idtmp}, Kitchen~\cite{pddlstream} and Rover~\cite{pddlstream}) visualized in Fig.~\ref{fig:teaser}. For an in-depth explanation of these domains, we refer the reader to the supplementary material on the \href{https://branvu.github.io/coast.github.io}{webpage}.  
We compare the cumulative median task and motion planning times of our algorithm with a 50\% confidence interval to that of PDDLStream's \texttt{Adaptive} Algorithm and IDTMP. For the \textit{Blocks} domain, we compare PDDLStream, IDTMP, and our method on problems with 0--6 obstructing blocks, with 5 seeded trials for each. We run IDTMP and our method with collision constraints. We further compare our method to PDDLStream in the \textit{Kitchen} and \textit{Rover} domains. For \textit{Kitchen}, we increase the difficulty by incrementing the number of cook and clean goals from 1 to 8, each involving a variable number of the 4 items. We use our general timestep constraint and 50 seeded trials for each number of goals. For \textit{Rover}, we scale the number of objectives and rocks (goal objects) $N$ from 1 to 4, using our action constraint with 10 seeded trials. We institute a total planning timeout of 1200s on all domains and trials and report the percentage solved on each trial. Since \textit{Kitchen} and \textit{Rover} have large state spaces, we only compare them against PDDLStream because IDTMP requires a discretization of the object state space, which is not a scalable approach for these two domains. 

\paragraph{Results}
The results for \textit{Blocks} are shown in Fig.~\ref{fig:blocks}. Overall, IDTMP's cumulative task planning time with the CSP solver is comparable in magnitude to our PDDL planning with some differences attributed to FastDownward \cite{helmert2006fast} IO overhead. This demonstrates that our PDDL constraint framework is comparable in performance to IDTMP's CSP constraint framework.  

PDDLStream, which also uses Fast Downward, significantly slows down with the number of blocks. The \textit{Blocks} domain has many infeasible actions due to obstructing blocks, which is difficult for PDDLStream to handle. PDDLStream requires many stream objects and iterations of planning to support long and geometrically feasible actions. For the five obstacle case, PDDLStream and IDTMP have a median time of 600s, whereas our method takes 6s to plan. 
For IDTMP, motion planning is the bottleneck. To find the goal configuration for an action, IDTMP samples various collision-free goal configurations around a target object or location until a timeout or collision-free configuration is found.
In contrast, in PDDLStream and our method, we calculate the inverse kinematics solution for the grasp and approach pose and perform collision-checking for the trajectory between them. This streamlines the motion planning process significantly. We express this motion planning grounding with actions, streams, and stream objects in PDDLStream's \texttt{stream.pddl} and our \texttt{geometric.pddl}.

The results for \textit{Kitchen} are shown in Fig.~\ref{fig:kitchen}. The \textit{Kitchen} task involves repetitive transfer of objects between surfaces, which requires long, chained stream plans where stream instance outputs are frequently inputs to future stream instances. For PDDLStream, This requires many iterations of stream generation to produce high-level stream instances. 
This slows down task planning because many symbolic stream objects are added to the task state.
When there are 8 goals, PDDLStream takes a total planning time of 1000s, whereas our method takes 10s. 
For the \textit{Rover} domain in Fig. \ref{fig:rover}, PDDLStream spends the least amount of time on task planning compared to \textit{Block} and \textit{Kitchen}. This is specifically because in this domain PDDLStream's incremental stream generation algorithm can recycle rover positions across actions and iterations of planning. This requires fewer iterations of stream generation and therefore fewer stream objects in the task state, resulting in faster planning. Since we directly ground an action plan into a stream plan, every stream object will be unique, which prevents the recycling of stream objects. We circumvent this issue by introducing a caching extension to streams. As shown in the ablation in Fig~\ref{fig:rover}, without caching, our method is less efficient than PDDLStream during motion planning because it cannot recycle stream results across a plan. However, with caching stream instances that have PDDL objects as inputs, we achieve similar performance. With four goal objects, our task planning time with caching is 1s compared to PDDLStream's 10s; however, our total time is only marginally better since motion planning time dominates in this domain.
Overall, we show superior performance compared to PDDLStream and IDTMP on total planning time for \textit{Blocks} and \textit{Kitchen} and superior task planning time to PDDLStream on \textit{Rover}.    
\paragraph{Discussion} \label{sec:limitations}
Our method requires a new formulation of writing how streams and PDDL actions are combined, but we believe this formulation is more straightforward and as expressive as PDDLStream from the results on three different domains. With this new formulation, we can remove the optimistic stream generation process and directly map action plans to stream plans. This comes at a cost of not being able to rely on task planning to recycle stream outputs across different actions of a task plan, making refinement more inefficient per plan. However, this motion planning inefficiency is insignificant compared to the performance boost gained by our task planning approach.
A limitation of constraints are that they require manual engineering for each task; however, this can be a benefit since it gives a way for the user to directly embed domain knowledge to reduce the search space of TAMP. 

\section{Conclusion}
\label{sec:conclusion}
We present COAST, a sampling-based TAMP algorithm that significantly outperforms previous state-of-the-art algorithms in terms of task planning time for a variety of domains. The key to faster planning is our novel stream planning subroutine, which bridges vanilla PDDL constraint-based task planning with stream-based motion planning and allows us to benefit from both. 
\section{Probabilistic Completeness}
\label{sec:prob-proof}
\begin{theorem}
For feasible problems, as the number of samples approach infinity, the probability of success of COAST will approach 1.
\end{theorem}
\begin{proof}
Given a TAMP task formulated as described in Sec.~\ref{sec:methods}, where the given task planner is complete and streams are probabilistically-complete, let a feasible task plan be $\pi_f$ and a feasible refinement for $\pi_f$ be $Y_f$. Our algorithm queries the task planner and then calls the semi-complete \texttt{Adaptive} sampling algorithm \cite{pddlstream}. \texttt{Adaptive} will eventually find a feasible refinement if it exists. If no refinement is found, we backtrack to a previous state, and attempt refinement again, continuing to attempt all unsuccessfully refined task plans. Since task planning is complete, we are guaranteed to produce $\pi_f$ and since we eventually reattempt all unsuccessful refinements, our algorithm will eventually find the feasible refinement $Y_f$ to $\pi_f$ and our algorithm is probabilistically-complete. 
\end{proof}



{\footnotesize
\bibliographystyle{IEEEtranN}
\bibliography{references}
}

\ifx\arxiv\undefined
\clearpage
\else
\fi

\appendix
\subsection{In-Depth Comparison to Prior Works}
\label{appx:prior}

\subsubsection{PDDLStream}

\begin{algorithm}[tb]
    \caption{PDDLStream algorithm overview}
    \begin{algorithmic}[1]
        \Function{PDDLStream}{$\Phi, \mathcal{A}, \Sigma, \mathcal{O}, s_0, g$}
            \While{not \Call{Timeout}{}}
                \State $\mathcal{O} \gets \mathcal{O} \cup \Call{CreateStreams}{\Sigma, \mathcal{O}}$ \label{line:streamobjects}
                \State $\pi \gets \Call{TaskPlan}{\Phi, \mathcal{A}, \mathcal{O}, s_0, g}$
                \If{$\pi = \text{None}$}
                    \textbf{continue}
                \EndIf
                \State $\psi \gets \Call{RetraceStreams}{\Sigma, \pi}$
                
                \State $Y, s_\psi \gets \Call{AdaptiveBinding}{\psi}$
                \If{\Call{IsSuccessful}{$\psi, s_\psi$}}
                    \Return $\pi_{\text{geom}}, Y$
                \EndIf
                \State $s_0 \gets \Call{UpdateCertifiedFacts}{\psi, Y, s_\psi, s_0}$
            \EndWhile
        \EndFunction
    \end{algorithmic}
    \label{alg:pddlstream}
\end{algorithm}

Alg.~\ref{alg:pddlstream} provides an overview of the PDDLStream algorithm~\cite{pddlstream}. The key difference between PDDLStream and our method is that PDDLStream defines PDDL actions using stream objects as parameters and certified facts as preconditions. In the following example, the PDDL \pddl{Place} action takes in stream objects \pddl{?q1}, \pddl{?g}, \pddl{?p}, \pddl{?q2}, and \pddl{?t} as inputs, and includes certified facts \pddl{sample-pose}, \pddl{sample-ik}, and \pddl{check-collision} as preconditions.
\begin{lstlisting}
(:action Place
  :parameters (
    ?o - obj ?r - region ?q1 - conf ?g - grasp
    ?p - pose ?q2 - conf ?t - traj)
  :precondition (and
    (in-grasp ?o ?g)
    (at-conf ?q1)
    (sample-pose ?o ?r ?p)
  :effect (and
    (on ?o ?r)
    (handempty)
    (not (in-grasp ?o ?g))
    (at-pose ?o ?p)
    (not (at-conf ?q1))
    (at-conf ?q2)))
\end{lstlisting}
The benefit of integrating certified facts into the preconditions of the PDDLStream \pddl{:action} is that there is no need to add constraints to the PDDL domain to force the task planner to produce alternative plans when motion planning is infeasible. When a stream instance fails, its level is increased, and it is not added to the task state in the next iteration and thus the preconditions for the corresponding action call will no longer be satisfied.

The main disadvantage of this approach is that the task planner is now required to decide what stream objects to use for a task plan. It is not known a priori what stream objects are required to satisfy the task planning goal. Therefore, task planning in PDDLStream is an iterative process where stream objects are incrementally introduced to the PDDL problem (Line~\ref{line:streamobjects}) until the task planner succeeds. At first, when no stream objects are available, the task planner will certainly fail. As the number of stream objects grows, task planning quickly becomes intractable due to its PSPACE-hard complexity. Task planning is therefore a significant bottleneck in PDDLStream when many stream objects are required, which may happen for problems that have many objects or that require long task plans.

Our key insight is that deciding what stream objects to use for a task plan can be done with a simple stream planning procedure that does not require a PDDL solver. COAST therefore alleviates the task planning bottleneck of PDDLStream by removing stream objects and certified facts from the task planning domain and handling them in a separate stream planning stage. Feedback from the motion planner is sent back to the task planner via PDDL constraints instead of certified facts.

\subsubsection{IDTMP}

IDTMP proposes a TAMP framework that, like our method, sends feedback from the motion planner to the task planner in the form of task planning constraints. However, IDTMP defines these constraints in a format that is only compatible with {\em Satisfiability Modulo Theory\/} (SMT) solvers. Our method, on the other hand, specifies constraints using vanilla PDDL and thus is compatible with both off-the-shelf PDDL solvers and SMT solvers. One distinct difference between IDTMP and our method is IDTMP incrementally expands the task horizon to fully constrain lower horizons before trying longer horizons. In our method, we do not have this restriction, and we can plan at any-length anytime.

\subsection{In-Depth Domain Descriptions}
\label{appx:in-depth-domains}
\subsubsection{Blocks}

This domain is adapted from IDTMP~\cite{idtmp}. A robot is given one red block and up to six blue obstructing blocks placed in random locations on a $3 \times 3$ grid. The goal is to perform a sequence of \pddlaction{Pick} and \pddlaction{Place} actions to move the red block into the center grid location without colliding with obstructing blocks, where the robot can only grasp blocks from the front. Since there are bound to be obstructions and limited solutions to clear them, this domain demonstrates high task-motion complexity, non-monotonicity, and infeasible actions. Many iterations of task and motion planning are required to solve problems in this domain. We use a domain-specific constraint similar to IDTMP's collision constraint and a collision-free stream like that of PDDLStream to ensure safe trajectories. 

\subsubsection{Kitchen}
This domain is adapted from PDDLStream~\cite{pddlstream}. There are five available actions: \pddlaction{Pick}, \pddlaction{Place}, \pddlaction{Cook}, and \pddlaction{Clean}.
The scene includes a \pddl{table} with a \pddl{stove}, a \pddl{sink}, and four items that are initially placed at one of four positions on the \pddl{table}. The goal is to have certain items be \pddl{cooked} or \pddl{cleaned}.
The domain exhibits non-monotonicity because an item might have to be moved more than once in the plan. Since the sink and stove are limited in area, the robot has to plan a sequence of actions with a constraint on the number of items on each platform at a time. This domain also features non-geometric actions with the use of the \pddl{Cook} and \pddl{Clean} actions. We use the timestep constraint for not reproducing the same plan if a stream within an action fails.

\subsubsection{Rover}
We adapt the Rover domain from PDDLStream where there are two controllable rovers, a stationary lander, eight obstacles, and $N$ rocks and $N$ objectives \cite{pddlstream}. The goal of this domain is to establish a line-of-sight communication channel between the rovers and the lander to transmit visual data of objectives and samples. The challenge is that the arena is disconnected and the initial state does not include which rovers can reach specific rocks or objectives which means this domain has many infeasible actions. This domain has high task-motion complexity since many iterations of task to motion planning are required to discover feasible plans. We use a parameter constraint to enforce that failed streams should prevent an action from occurring again over all steps of a plan. We also experimented with caching stream instances during sampling with some probability parameter which we ablate in Fig. \ref{fig:rover}.
\subsection{Inverse Kinematics and Collision Checking}
\label{appx:geometric-action}

Below, we show a version of the geometric action \pddlaction{Place} defined in Sec.~\ref{sec:gpddl-stream-planning} that considers robot kinematics and collisions:
\begin{lstlisting}[label=code:gpddl]
(:geom-action Place
  :parameters (?o - obj ?r - region)
  :inputs (?q1 - conf ?g - grasp)
  :outputs (?p - pose ?q2 - conf ?t - traj)
  :geom-precondition (and
    (in-grasp ?o ?g)
    (at-conf ?q1)
    (sample-pose ?o ?r ?p)
    (sample-ik ?o ?p ?g ?q2 ?t)
    (forall ?oo - obj
      (forall ?pp - pose
        (when
          (and
            (at-pose ?oo ?pp)
            (not (= ?oo ?o)))
          (check-collision ?t ?oo ?pp)))))
  :geom-effect (and
    (not (in-grasp ?o ?g))
    (at-pose ?o ?p)
    (not (at-conf ?q1))
    (at-conf ?q2)))
\end{lstlisting}

This \pddlaction{Place} action outputs a target joint configuration \pddl{?q2} and a trajectory \pddl{?t} to perform the placing action via a \pddl{sample-ik} stream. Additionally, the geometric preconditions for this action require that this trajectory does not collide with any other object \pddl{?oo}, which is checked via the \pddl{check-collision} stream.

For comparison, we provide below PDDLStream's action definition which contains inverse-kinematics and collision-checking certified facts.
\begin{lstlisting}
(:action Place
  :parameters (
    ?o - obj ?r - region ?q1 - conf ?g - grasp
    ?p - pose ?q2 - conf ?t - traj)
  :precondition (and
    (in-grasp ?o ?g)
    (at-conf ?q1)
    (sample-pose ?o ?r ?p)
    (sample-ik ?o ?p ?g ?q2 ?t)
    (forall ?oo - obj
      (forall ?pp - pose
        (when
          (and
            (at-pose ?oo ?pp)
            (not (= ?oo ?o)))
          (check-collision ?t ?oo ?pp)))))
  :effect (and
    (on ?o ?r)
    (handempty)
    (not (in-grasp ?o ?g))
    (at-pose ?o ?p)
    (not (at-conf ?q1))
    (at-conf ?q2)))
\end{lstlisting}

\subsection{Sequence Constraints}
\label{appx:timestep}

Our method automatically augments PDDL action definitions with timestamps to allow the task planner to avoid proposing sequences of actions that have led to motion planning failures.
Specifically, two timestamp parameters, \pddl{?tprev} and \pddl{?t}, are added to each action, where \pddl{?tprev} is the timestamp before the action is executed and \pddl{?t} is the timestamp after, enforced by preconditions \pddl{(at-time ?tprev)} and \pddl{(next-time ?tprev ?t)}. Below, we show the augmented definitions of \pddl{Pick} and \pddl{Place}, where the original preconditions and effects from the definitions in Sec.~\ref{code:pddl} are replaced with {\color{mygreen} \pddl{; ...same as original}} for brevity.
\begin{lstlisting}
(:action Pick
  :parameters (
    ?o - obj ?r - region ?tprev - time ?t - time)
  :precondition (and
    ; ...same as original
    (at-time ?tprev)
    (next-time ?tprev ?t)
    (not (fail-pick ?o ?r ?tprev ?t)))
  :effect (and
    ; ...same as original
    (not (at-time ?tprev))
    (at-time ?t)
    (log-pick ?o ?r ?tprev ?t)))
(:action Place
  :parameters (
    ?o - obj ?r - region ?tprev - time ?t - time)
  :precondition (and
    ; ...same as original
    (at-time ?t1)
    (next-time ?tprev ?t)
    (not (fail-place ?o ?r ?tprev ?t)))
  :effect (and 
    ; ...same as original
    (not (at-time ?tprev))
    (at-time ?t)
    (log-place ?o ?r ?tprev ?t)))
\end{lstlisting}
Apart from timestamp boilerplate, each action is also automatically augmented with a symbolic effect that logs the execution of an action during task planning. For example, \pddl{(log-pick apple table t0 t1)} is added to the symbolic state after the execution of \pddlaction{Pick(apple, table, t0, t1)}. Put another way, if a symbolic state contains \pddl{(log-pick apple table t0 t1)} during task planning, then this indicates that \pddlaction{Pick(apple, table, t0, t1)} was executed by the task planner at timestamp \pddl{t1}.

Suppose the task planner proposes a simple four-step plan: [\pddlaction{Pick(apple, table, t0, t1)}, \pddlaction{Place(apple, rack, t1, t2)}, \pddlaction{Pick(banana, rack, t2, t3)}, \pddlaction{Place(banana, table, t3, t4)}]. If the motion planning for this action skeleton fails on the third action \pddlaction{Pick(banana, rack, t2, t3)}, then the PDDL domain will be updated to prevent the first three actions from being planned in sequence again. Specifically, an effect is added to the action preceding the failed action such that if \pddl{(log-pick apple table t0 t1)} has been added to the symbolic state and \pddlaction{Place(apple, rack, t1, t2)} is executed by the task planner, then \pddl{(fail-pick banana rack t2 t3)} is added to the symbolic state. Below, we show this added effect, with the original effects from the timestamp-augmented definition replaced with the comment {\color{mygreen} \pddl{; ...same as augmented}}.
\begin{lstlisting}
(:action Place
  :effect (and 
    ; ...same as augmented
    (when
      (and
        (log-pick apple table t0 t1)
        (= ?o apple)
        (= ?r rack)
        (= ?tprev t1)
        (= ?t t2))
      (fail-pick banana rack t2 t3))))
\end{lstlisting}
The added proposition \pddl{(fail-pick banana rack t2 t3)} prevents \pddlaction{Pick(banana, rack, t2, t3)} from being executed, since \pddl{(not (fail-pick banana rack t2 t3))} is a precondition of this action. In this case, the task planner will be forced to find a different sequence of actions to satisfy the goal.

It is important to note that \pddl{(fail-pick banana rack t2 t3)} can only be added to the state when the task planner is considering [\pddlaction{Pick(apple, table, t0, t1)}, \pddlaction{Place(apple, rack, t1, t2)}] as the first two actions in the task plan. Otherwise, this proposition may invalidate an unrelated action sequence such as [\pddlaction{Pick(orange, rack, t0, t1)}, \pddlaction{Place(orange, table, t1, t2)}, \pddlaction{Pick(banana, rack, t2, t3)}], which may have a chance of succeeding at motion planning.

To formalize this failure constraint, let us represent the $i$-th action in an action skeleton as $a_i(x_1^i, \dots, x_m^i)$, where $x_1^i, \dots, x_m^i$ are the $m$ parameters for action $a_i$ (e.g. \pddl{?o}, \pddl{?r}, \pddl{?tprev}, and \pddl{?t} in \pddlaction{Pick(?o, ?r, ?tprev, ?t)}.) Let $o_1^i, \dots, o_m^i$ be object assignments for these parameters, such that $a_i(o_1^i, \dots, o_m^i)$ represents an action call such as \pddlaction{Pick(apple, table, t0, t1)}. Let \pddl{log-}$a_i$ and \pddl{fail-}$a_i$ represent the logging and failure predicates corresponding to action $a_i$. If motion planning fails on action $t$, then we add the following effect to action $a_{t-1}$:
\begin{align}
    \begin{split}
        &\text{\pddl{log-}}a_1(o_1^1, \dots, o_m^1) \land \dots \land \text{\pddl{log-}}a_{t-2}(o_1^{t-2}, \dots, o_m^{t-2}) \\
        &\quad\land x_1^{t-1} = o_1^{t-1} \land \dots \land x_m^{t-1} = o_m^{t-1} \\
        &\quad\implies \text{\pddl{fail-}}a_t(o_1^t, \dots, o_m^t).
    \end{split} \label{eq:timestamp-constraint}
\end{align}
Note that for action $a_{t-1}$, \pddl{log-}$a_{t-1}(o_1^{t-1}, \dots, o_m^{t-1})$ has not yet been added to the symbolic state, so it is necessary to check that the input parameters $x_1^{t-1}, \dots, x_m^{t-1}$ are equivalent to the objects $o_1^{t-1}, \dots, o_m^{t-1}$ in the failed action sequence.

If stream sampling fails for the first action $a_1$, then $\text{\pddl{fail-}}a_1(o_1^t, \dots, o_m^t)$ is simply added to the initial state.
\subsection{Collision-Generalization Constraint}
To support collision-generalization constraints, we add an optional \pddl{:fail-effect} field to streams that specifies an effect that should be added to the effect of the parent PDDL action upon failing, similar to the conditional effect for timestamps in Eq.~\ref{eq:timestamp-constraint}. For example, consider the following \pddl{check-block-collision} stream, which checks whether a collision occurs with block \pddl{?b2} at grid location \pddl{?l2} while trying to execute trajectory \pddl{?t} with target location \pddl{?l1}.
\begin{lstlisting}
(:stream check-block-collision
  :inputs (?t - traj ?l1 - gridloc
    ?b2 - block ?l2 - gridloc)
  :fail-effect (and
    (when (not (clear ?l2))
      (blocked ?l1))
    (when (clear ?l2)
      (not (blocked ?l1)))))
\end{lstlisting}

Suppose the grounded task plan is [\pddlaction{Pick(b1, l9; q0, g1, q1, t1)}, \pddlaction{Place(b1, l1; q1, g1, q2, t2)}], and while stream sampling for \pddlaction{Place(b1, l1; q1 g1 q2 t2)}, the \pddl{check-block-collision(t2, l1, b2, l2)} stream instance fails (indicating a collision because of block \pddl{b2} at location \pddl{l2}). Then the specified \pddl{:fail-effect} gets added to the effects of the previous \pddl{Pick} action, where if \pddl{l2} is not clear, then \pddl{(blocked l1)} is added to the symbolic state. The converse effect is also added, where if \pddl{l2} becomes clear, then \pddl{(blocked l1)} is removed from the symbolic state.

\end{document}